\documentclass{article}
\usepackage{ICASSP2027_Paper_Templates/spconf,amsmath,amssymb,amsthm,graphicx,booktabs,array,float}

\usepackage{xcolor}
\floatstyle{ruled}
\newfloat{algorithm}{tbp}{loa}
\floatname{algorithm}{Algorithm}
\newenvironment{algorithmic}[1][0]{\begin{list}{}{\setlength{\leftmargin}{0pt}\setlength{\itemsep}{1pt}\setlength{\parsep}{0pt}\setlength{\topsep}{0pt}}}{\end{list}}
\newcommand{\REQUIRE}{\item[]\textbf{Input:}}
\newcommand{\ENSURE}{\item[]\textbf{Output:}}
\newcommand{\STATE}{\item[]}
\newcommand{\FOR}[1]{\item[]\textbf{for} #1 \textbf{do}\begin{list}{}{\setlength{\leftmargin}{1.2em}\setlength{\itemsep}{1pt}\setlength{\parsep}{0pt}\setlength{\topsep}{0pt}}}
\newcommand{\ENDFOR}{\end{list}\item[]\textbf{end for}}

\newcommand{\R}{\mathbb{R}}
\newcommand{\w}{\mathbf{w}}
\newcommand{\p}{\mathbf{p}}
\newcommand{\e}{\mathbf{e}}
\newcommand{\loss}{\boldsymbol{\ell}}

\theoremstyle{definition}
\newtheorem{assumption}{Assumption}
\theoremstyle{plain}
\newtheorem{proposition}{Proposition}

\title{Preference-Aware Pareto Steering for Multi-Task Speech LLMs}
\name{\small Liuyuan Jiang$^{1}$, Vishal Sunder$^{2}$, Samuel Thomas$^{2}$, Brian Kingsbury$^{2}$, Jeff Kuo$^{2}$, Lisha Chen$^{1}$, and George Saon$^{2}$}
\address{$^{1}$Department of Electrical and Computer Engineering, University of Rochester, Rochester, NY, USA\\
$^{2}$IBM Research, Thomas J. Watson Research Center, Yorktown Heights, NY, USA}

\begin{document}
\ninept
\maketitle
\begin{abstract}
Speech large language models (LLMs) perform multiple tasks, including automatic speech recognition (ASR), timestamp prediction (TS), and speech translation (ST). Joint training of speech LLMs creates task trade-offs, yet standard training with fixed loss weights does not generally control the resulting trade-off according to deployment preferences. To address this, we propose Endpoint-Anchored Preference-Aware Sampling Uniformly along the Pareto Front (EA-SURF), an a priori method that adapts scalarization weights so that the resulting task trade-off reflects the desired preference. EA-SURF retains the standard weighted-loss training pipeline, tracking the current model's performance relative to each task-specific reference and adapting the scalarization weights as training proceeds. We theoretically show that EA-SURF can match a desired preference, with the update becoming stationary at the target under checkable regularity conditions. Across six three-objective speech LLM requests, EA-SURF empirically converges toward the requested targets and increases endpoint-inclusive normalized hypervolume by 17.8\% over fixed loss weighting.
\end{abstract}

\section{Introduction}
\label{sec:intro}
Large language models (LLMs) have expanded beyond text to process and generate information across modalities. Speech LLMs combine acoustic representations with pretrained LLMs to support diverse speech and audio capabilities~\cite{tang2024salmonn,hu2024wavllm}. Existing systems span speech recognition and translation~\cite{ao2022speecht5,radford2023robust}, while recent models extend this interface to auditory instruction following and fine-grained capabilities such as word-level timing~\cite{tang2024salmonn,fan2026sync}. These developments motivate general-purpose speech LLMs that consolidate multiple capabilities in a shared model.

Training these models requires balancing multiple task objectives~\cite{saif2024m2asr}. Multi-objective optimization (MOO) provides a principled formulation~\cite{mahapatra2020multi,chen2024ferero,saif2024m2asr}:
\begin{equation}
\min_{\theta}\ \loss(\theta)
:=\big(\ell_1(\theta),\ldots,\ell_M(\theta)\big)\in\R^M,
\label{eq:moo}
\end{equation}
where $\theta$ denotes the trainable parameters and $M$ the number of objectives. MOO seeks \emph{Pareto-optimal} solutions, for which no objective can be improved without worsening another~\cite{miettinen1999nonlinear}. Their objective values form the \emph{Pareto front} (PF) of attainable trade-offs. For $m\in[M]$, let $\mathcal D_m$ denote its sampling distribution and $\mathcal L_m$ its per-sample objective, giving $\ell_m(\theta):=\mathbb E_{\xi\sim\mathcal D_m}[\mathcal L_m(\theta,\xi)]$. This formulation covers supervised losses and, after sign reversal, expected rewards estimated from sampled trajectories~\cite{jiang2026surf} in reinforcement learning problems.

Existing speech systems acquire multiple capabilities through staged adaptation, mixed-task supervision, or prescribed training recipes~\cite{tang2024salmonn,fan2026sync}. While effective, these approaches do not provide calibrated control over the attained trade-off. Preference-guided MOO distinguishes \emph{a priori} steering, where a known deployment preference specifies an operating point~\cite{mahapatra2020multi,chen2024ferero}, from \emph{a posteriori} optimization, which generates Pareto solutions for later selection~\cite{miettinen1999nonlinear,navon2021learning}. The former suits known deployment priorities, whereas the latter supports priorities chosen after training~\cite{miettinen1999nonlinear}. For bi-objective \emph{a posteriori} coverage, SURF~\cite{jiang2026surf} calibrates scalarization weights using PF geometry. Calibrating a prescribed preference for general $M\geq2$ objectives remains underexplored.

\begin{figure}[t]
\centering
\includegraphics[width=0.8\linewidth]{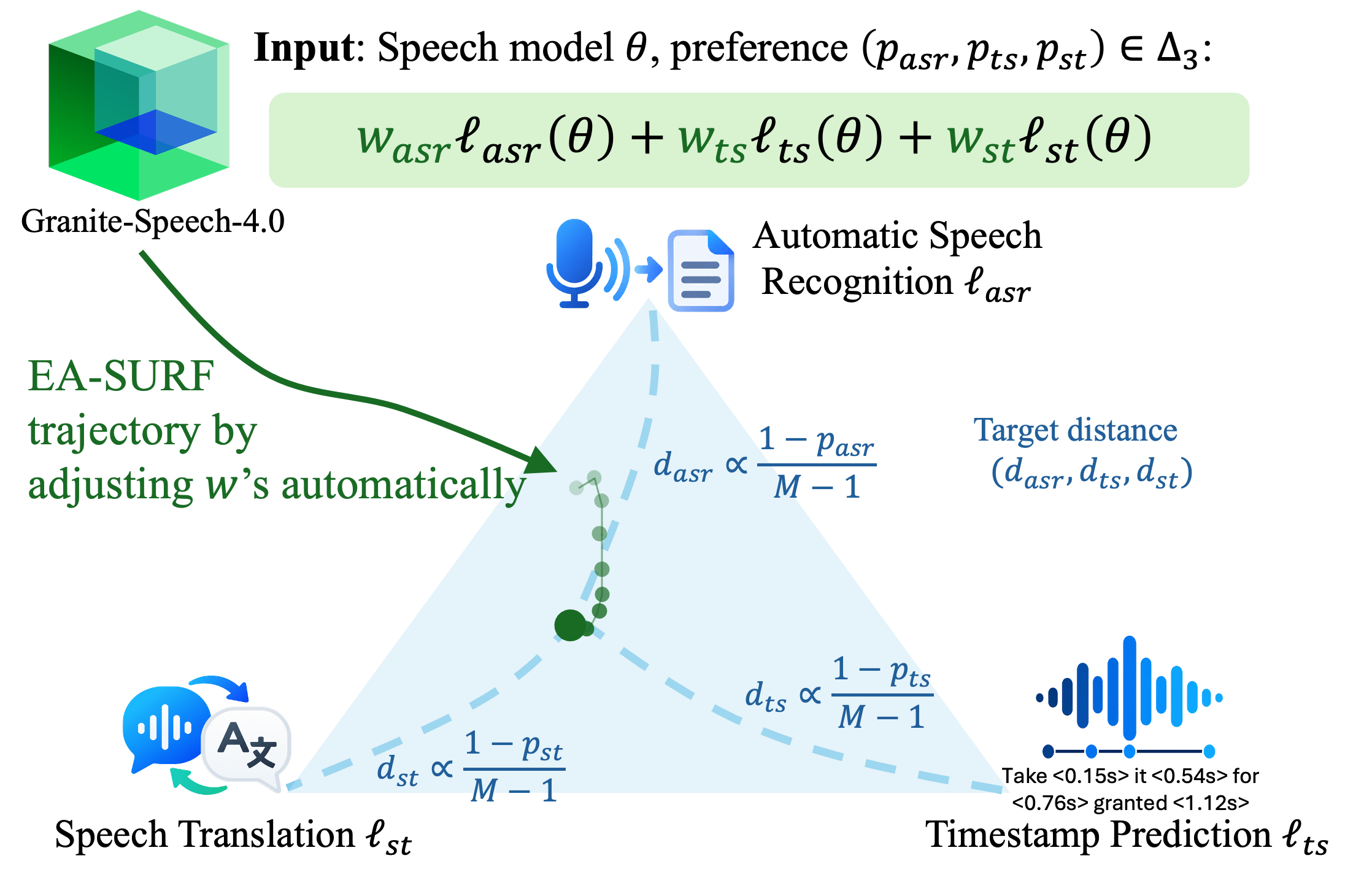}
\vspace{-0.2cm}
\caption{EA-SURF preference steering for ASR, timestamp prediction (TS), and speech translation (ST). Specialist models anchor the objective-space geometry, while a deployment preference $\p$ is mapped by \eqref{eq:target-share} to target specialist-relative distance shares $\mathbf q^\star(\p)$. EA-SURF iteratively adapts the scalarization weights $\w$ from the measured share mismatch and continues training from the current model toward the requested operating point.
}
\label{fig:controller}
\end{figure}

We propose EA-SURF for \emph{a priori} preference calibration. It uses the standard scalarization interface while adapting its training weights from measured PF geometry. Task sampling selects observations for each objective estimate, while scalarization determines how the objectives update the shared model. A standard choice is linear scalarization (LS)~\cite{miettinen1999nonlinear}:
\vspace{-0.2cm}
\begin{align}
\theta^\star(\w)\in\arg\min_\theta~\{\mathcal L_{\mathrm{LS}}(\theta,\w)=\sum_{m=1}^{M}w_m\ell_m(\theta)\}
\label{eq:ls}
\end{align}
where $\w\in\Delta_M$ weights the training objectives. Setting $\w=\p$ for a deployment preference $\p\in\Delta_M$ is natural, but it does not generally ensure the corresponding trade-off~\cite{miettinen1999nonlinear,jiang2026surf}. Indeed, uniformly spaced weights can produce operating points clustered in some regions of objective space~\cite{jiang2026surf}. Here, $\p$ encodes relative task priority rather than prescribed WER, AAS, or BLEU values. EA-SURF realizes it through target proportions of specialist-relative training-objective distances.

In the speech application of Fig.~\ref{fig:controller}, $M=3$, with automatic speech recognition (ASR), timestamp prediction (TS), and speech translation (ST), so $\loss(\theta)=(\ell_{\rm asr}(\theta),\ell_{\rm ts}(\theta),\ell_{\rm st}(\theta))$ and $\w=(w_{\rm asr},w_{\rm ts},w_{\rm st})$. For example, the deployment preference $(p_{\rm asr},p_{\rm ts},p_{\rm st})=(.4,.4,.2)$ prioritizes ASR and TS, yet the same proportions used as loss weights need not realize the intended compromise.
EA-SURF addresses this mismatch by fixing $\p$, adapting $\w$ from objective-space feedback, and continuing the same warm-started probe. Motivated by preference-ray Pareto representations~\cite{mahapatra2020multi}, we assign more important objectives smaller target distances from their specialists:
\begin{equation}
d_m(\theta):=\|\loss(\theta)-\e_m\|_2
\propto
q_m^{\star}(\p):=
\frac{1-p_m}{\sum_{j=1}^{M}(1-p_j)}.
\label{eq:target-share}
\end{equation}
where $\e_m=(e_{m,1},\ldots,e_{m,M})$ denotes the measured loss vector of a model trained predominantly for objective $m$. EA-SURF normalizes the measured distances into attained shares and compares them with the target shares $\mathbf q^\star(\p)$. This request maps to $(q_{\rm asr}^\star,\allowbreak q_{\rm ts}^\star,\allowbreak q_{\rm st}^\star)=(.3,.3,.4)$. Independent runs over a preference grid also yield multiple calibrated points for \emph{a posteriori} use. Our contributions are:
\begin{itemize}
\item[C1)] \textbf{Preference realizability theory.} Under regularity conditions, we show that every interior preference target can be matched by a unique scalarization weight, yielding a Pareto-optimal solution (Sec.~\ref{sec:formulation}).
\item[C2)] \textbf{Closed-loop algorithm.} EA-SURF adapts one warm-started probe using reusable specialists, with interior fixed points exactly matching the requested target shares (Sec.~\ref{sec:easurf}).
\item[C3)] \textbf{Speech experimental results.} Across six requests, EA-SURF improves preference alignment, lowers mean timestamp error from 46.6 to 39.2~ms, and raises normalized hypervolume by 17.8\% over fixed LS. Table~\ref{tab:main} reports the associated WER, AAS, and BLEU trade-offs (Sec.~\ref{sec:experiments}).
\end{itemize}

\noindent\textbf{Notation.} For $M\in\mathbb N_+$, let $[M]:={1,\ldots,M}$ and $\Delta_M:={\mathbf v\in\R_+^M:\sum_{m=1}^M v_m=1}$. Bold lowercase letters denote vectors. $\|\cdot\|_1$ and $\|\cdot\|_2$ are the $\ell_1$ and Euclidean norms. For a set $A$, $\operatorname{relint}(A)$ and $\partial A$ denote its relative interior and boundary. Subscripts index objectives, parenthesized superscripts index outer rounds, and $\star$ denotes optimal or target quantities.

\section{Related Work}
\label{sec:related}

\noindent\textbf{Unified and multi-task speech learning.} Modern speech models increasingly place several capabilities in one shared architecture. SpeechT5 uses multi-task pretraining~\cite{ao2022speecht5}, Whisper uses large-scale multitask supervision~\cite{radford2023robust}, and SALMONN and WavLLM use cross-modal alignment and instruction tuning to broaden speech and audio interfaces~\cite{tang2024salmonn,hu2024wavllm}. In-Sync adds timestamp prediction through task conditioning and timestamp regularization~\cite{fan2026sync}. More explicitly multi-objective speech work includes M2ASR, which optimizes multilingual ASR and ST under unsupervised-performance constraints~\cite{saif2024m2asr}, and Objective Soups, which combines conflict-aware gradients for multilingual ASR and ST~\cite{saif2026objective}. These approaches focus on shared architectures, supervision, or joint optimization. EA-SURF instead uses attained specialist-relative shares to calibrate one fixed deployment preference during continued adaptation.

\vspace{0.1cm}
\noindent\textbf{Multi-task and preference-aware optimization.} Multi-task optimization also acts through controls distinct from a requested Pareto location. PCGrad modifies task gradients to reduce interference~\cite{yu2020gradient}, whereas PiKE adapts task proportions from gradient interactions~\cite{li2026pike}. These methods respectively alter the loss update and data exposure. In our formulation, the sampler remains fixed while EA-SURF adapts scalarization weights from objective-space feedback. Fixed LS is therefore the main experimental control because EA-SURF retains the same weighted-loss, single-model training interface, whereas gradient-combination approaches require per-objective gradient calculations not used here. Preference-guided MOO provides complementary ways to state a desired compromise. EPO targets a loss-space ray~\cite{mahapatra2020multi}, and FERERO encodes relative or absolute constraints~\cite{chen2024ferero}. For \emph{a posteriori} selection, 
Pareto HyperNetworks learn preference-conditioned solution families~\cite{navon2021learning}, while SURF uses arc-length geometry to improve $M=2$ PF coverage~\cite{jiang2026surf}. EA-SURF instead calibrates the scalarization weights for one known request by comparing attained specialist-relative shares with the request-derived target along the same continuing probe.

\section{EA-SURF Method}

\subsection{Preference-Steering Geometry}
\label{sec:formulation}

The deployment preference $\p$ specifies the requested trade-off. The scalarization weight $\w$ controls training, and the two need not coincide. We characterize the operating point attained by $\w$ through the specialist distances in \eqref{eq:target-share}. For $\sum_jd_j(\theta)>0$, let $q_m(\theta):=d_m(\theta)/\sum_{j=1}^{M}d_j(\theta)$ and $\mathbf q(\theta):=(q_1(\theta),\ldots,q_M(\theta))$. Thus, $\mathbf q(\theta^\star(\w))$ records the attained position relative to the specialists. The target $\mathbf q^\star(\p)$ in \eqref{eq:target-share} is the unique affine, permutation-equivariant, specialist-consistent mapping from preferences to distance shares. For a pure preference for objective $m$, specialist consistency assigns zero share to its own specialist. Symmetry and simplex normalization assign $1/(M-1)$ to the remaining specialists, and affinity gives $q_m^\star(\p)=(1-p_m)/(M-1)$. Preference steering therefore seeks $\w^\star\in\operatorname{relint}(\Delta_M)$ such that $\mathbf q(\theta^\star(\w^\star))=\mathbf q^\star(\p)$. The following assumption makes this correspondence well posed.

\begin{assumption}[Regular preference matching]
\label{ass:pf-regularity}
The objectives admit a hidden representation satisfying the hidden-convexity, smoothness, and PF nondegeneracy conditions in~\cite{miettinen1999nonlinear,jiang2026surf}. The attained shares $\mathbf q(\theta^\star(\w))$ are continuously differentiable in $\w\in\operatorname{relint}(\Delta_M)$ with Jacobian rank $(M-1)$. For each interior preference considered, $\|\mathbf q(\theta^\star(\w))-\mathbf q^\star(\p)\|_2^2$ is strictly convex on $\Delta_M$ and attains its minimum in $\operatorname{relint}(\Delta_M)$.
\end{assumption}

Assumption~\ref{ass:pf-regularity} gives sufficient conditions for a unique preference match. Hidden convexity and PF regularity support scalarization, differentiability and full rank ensure that the attained preference responds locally to changes in scalarization weights, and strict convexity with interior attainment ensures uniqueness. Cross-entropy is convex in the output logits and, under a fixed linear output head, in the final hidden representation, despite being generally nonconvex in the full model parameters~\cite{jiang2026surf}. We empirically examine whether the resulting trajectories remain well behaved and whether multiple mismatch minima arise in Sec.~\ref{sec:experiments}.

\begin{proposition}[Preference realizability]
\label{prop:surjectivity}
Under Assumption~\ref{ass:pf-regularity}, every $\p\in\operatorname{relint}(\Delta_M)$ under consideration admits a unique $\w^\star\in\operatorname{relint}(\Delta_M)$ such that $\mathbf q(\theta^\star(\w^\star))=\mathbf q^\star(\p)$. Consequently, $\theta^\star(\w^\star)$ is Pareto optimal and satisfies $d_m(\theta^\star(\w^\star))\propto (1-p_m)$, $\forall m\in[M]$.
\end{proposition}
\begin{proof}
Assumption~\ref{ass:pf-regularity} gives a unique interior minimizer of the squared mismatch. First-order optimality makes its residual orthogonal to the directions generated by the attained-share Jacobian. Both share vectors lie in $\Delta_M$, so their residual has zero sum. The rank-$(M-1)$ Jacobian spans this $(M-1)$-dimensional simplex tangent space, which forces the residual to be zero. Strict convexity gives the unique matching weight. Hidden convexity and positive interior weights make the associated LS solution locally unique and Pareto optimal~\cite{miettinen1999nonlinear,jiang2026surf}.
\end{proof}

Proposition~\ref{prop:surjectivity} establishes existence and uniqueness of the matching supported solution under Assumption~\ref{ass:pf-regularity}. It does not establish convergence of the finite-horizon adaptive update.

\subsection{EA-SURF: Adaptive Scalarization from Preference}
\label{sec:easurf}

\begin{algorithm}[t]
\caption{EA-SURF for one deployment request}
\label{alg:easurf}
\begin{algorithmic}[1]
\REQUIRE Preference $\p$, anchors ${\e_m}_{m=1}^{M}$, warm start $\theta^{(0)}$, gain $\alpha$, optimizer $\mathrm{TrainStep}$, budgets ${K_k}_{k=0}^{T-1}$. Set $\mathbf q^\star(\p)$ by \eqref{eq:target-share} and $\w^{(0)}\gets\p$
\FOR{$t=0,\ldots,T-1$}
\STATE Measure $\loss(\theta^{(t)})$, compute $d_m(\theta^{(t)})$, and normalize to $\mathbf q(\theta^{(t)})$
\STATE Form the current-to-target share ratios and update $\w^{(t+1)}$ by \eqref{eq:easurf}
\STATE Set $\theta^{(t,0)}\gets\theta^{(t)}$
\FOR{$k=0,\ldots,K_t-1$}
\STATE $\theta^{(t,k+1)}\gets\mathrm{TrainStep}(\theta^{(t,k)},\w^{(t+1)})$
\ENDFOR
\STATE Set $\theta^{(t+1)}\gets\theta^{(t,K_t)}$
\ENDFOR
\ENSURE Terminal probe $\theta^{(T)}$
\end{algorithmic}
\end{algorithm}

Proposition~\ref{prop:surjectivity} establishes a unique supported realization of the target shares. EA-SURF seeks this realization by adapting LS weights from observed share mismatch while continuing a single warm-started probe. At outer iteration $t$, the preference calibration error (PCE), given by $\operatorname{PCE}^{(t)}(\p):=\|\mathbf q(\theta^{(t)})-\mathbf q^{\star}(\p)\|_1$, measures the mismatch between attained and target shares.
Starting from $t=0$ with $\w^{(0)}=\p$, EA-SURF forms $r_m^{(t)}:=(q_m(\theta^{(t)})+\varepsilon_{\mathrm U})/(q_m^{\star}(\p)+\varepsilon_{\mathrm U})$ for $m\in[M]$ and applies the exponentiated, simplex-preserving update~\cite{kivinen1997exponentiated}
\begin{equation}
w_m^{(t+1)}:=\frac{w_m^{(t)}\big(r_m^{(t)}\big)^{\alpha}}
{\sum_j w_j^{(t)}\big(r_j^{(t)}\big)^{\alpha}},
\label{eq:easurf}
\end{equation}
where $0<\alpha\leq1$ and $\varepsilon_{\mathrm U}>0$ keeps the ratios finite. When $q_m(\theta^{(t)})/q_m^\star(\p)>1$, the probe is relatively too far from specialist $m$. EA-SURF therefore increases the relative LS weight on objective $m$, with the intended effect of moving the next continuation closer to that specialist. The same probe then trains for $K_t$ steps under the updated weights, where $\mathrm{TrainStep}(\theta,\w)$ denotes one optimizer update on the weighted objective $\mathcal L_{\mathrm{LS}}(\theta,\w)$ in~\eqref{eq:ls}.
Algorithm~\ref{alg:easurf} summarizes this closed-loop procedure. Each request uses one warm-started probe and reuses the $M$ specialist anchors, avoiding neighboring simplex probes or an explicit PF mesh. Computing $M$ anchor distances and updating $M$ weights adds negligible overhead relative to the $K_t$ training steps.

\begin{figure}[!t]
\centering
\includegraphics[width=0.85\linewidth]{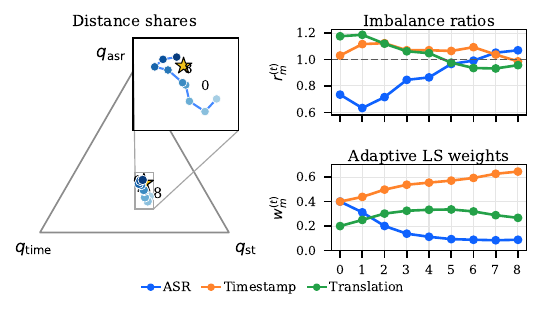}
\vspace{-0.6cm}
\caption{EA-SURF dynamics for $\p=(.4,.4,.2)$. Left: attained shares in the simplex. Upper right: current-to-target ratios. Lower right: the resulting adaptive LS weights. The gold star marks $\mathbf q^\star(\p)=(.3,.3,.4)$, and the dashed line marks ratio one.}
\label{fig:easurf-dynamics}
\vspace{-0.3cm}
\end{figure}

Figure~\ref{fig:easurf-dynamics} illustrates the feedback dynamics: measured shares move toward their targets and imbalance ratios move toward one, while LS weights depart substantially from the original preference $\p$. Thus, matching the requested trade-off can require training weights different from the preference itself. The next result characterizes the update's fixed point.

\begin{proposition}[Target consistency]
\label{prop:target-consistency}
For $\w^{(t)}\in\operatorname{relint}(\Delta_M)$ and $\alpha>0$, the update in \eqref{eq:easurf} satisfies $\w^{(t+1)}=\w^{(t)}$ if and only if $\mathbf q(\theta^{(t)})=\mathbf q^\star(\p)$.
\end{proposition}
\begin{proof}
At an interior fixed point, all $(r_m^{(t)})^\alpha$, and hence all $r_m^{(t)}$, equal some $c>0$. Therefore $q_m(\theta^{(t)})+\varepsilon_{\mathrm U}=c(q_m^\star(\p)+\varepsilon_{\mathrm U})$. Summing over $m$ and using $\sum_m q_m(\theta^{(t)})=\sum_m q_m^\star(\p)=1$ yields $c=1$. The converse follows immediately from \eqref{eq:easurf}.
\end{proof}
Thus, the desired preference match is exactly the interior stationary condition of the weight update. This target-consistency result is not a convergence theorem. Finite-horizon behavior is assessed through the trajectories in Sec.~\ref{sec:experiments}.

\section{Experiments}
\label{sec:experiments}

\subsection{Speech LLM, tasks, and evaluation}

\noindent\textbf{Configuration and model.} We use the SLAM-LLM speech interface~\cite{ma2024embarrassingly} adopted by Granite-Speech~\cite{saon2025granite} and In-Sync~\cite{fan2026sync}, with a frozen ConfCTC encoder, a trainable Q-Former speech adapter~\cite{li2023blip}, and a frozen Granite-4.0-1B-Base causal LLM. Audio is sampled at 16~kHz, capped at 60~s, and converted to 80-bin log-Mel features. We train the speech adapter and rank-64 LoRA modules~\cite{hu2022lora}, with scaling 32 and dropout 0.1, for 81M trainable parameters. Training uses four NVIDIA H100 GPUs with a total batch size of 64.

\vspace{0.1cm}
\noindent\textbf{Training and validation data.} ASR training combines Common Voice and LibriSpeech~\cite{ardila2020common,panayotov2015librispeech}, VoxPopuli, Earnings-22, the AMI Meeting Corpus (AMI), CALLHOME, Switchboard, Fisher, and additional voicemail, multispeaker, and synthetic speech. TS uses forced-aligned subsets of these corpora. We discard samples whose alignment score between CTC forced alignment and GMM-HMM alignment from the Montreal Forced Aligner exceeds 175~ms~\cite{mcauliffe2017montreal}. Following In-Sync~\cite{fan2026sync}, targets interleave words with centisecond end times. ST uses nine Common Voice-derived translation sets spanning English-to-multilingual and multilingual-to-English directions. Validation draws from Multilingual LibriSpeech (MLS), Common Voice, LibriSpeech, VoxPopuli, AMI, and Earnings-22.

\vspace{0.1cm}
\noindent\textbf{Evaluation.} We evaluate ASR on seven held-out English conditions and TS on nine aligned conditions, including manually timestamped TIMIT~\cite{garofolo1993timit} and Buckeye~\cite{pitt2005buckeye}. TIMIT and Buckeye are excluded from training and therefore provide out-of-domain timestamp evaluations. ST uses thirteen directions from CoVoST~2 and FLEURS~\cite{wang2021covost,conneau2023fleurs}. We report unweighted macro-averages using word error rate (WER), accumulated averaging shift (AAS)~\cite{shi2022achieving}, and bilingual evaluation understudy (BLEU) computed with SacreBLEU~\cite{post2018call}. MOO performance is measured by PCE and endpoint-inclusive normalized hypervolume (HV)~\cite{zitzler1999multiobjective}. PCE is an internal preference-alignment metric because it uses the same training-objective share mismatch that drives EA-SURF. WER, AAS, BLEU, and HV assess held-out outcomes or trade-off coverage. For HV, we minimize $(\mathrm{WER}_{\mathrm{ASR}},\mathrm{AAS},-\mathrm{BLEU})$, normalize each coordinate by the ideal and nadir of the three shared anchors, and use the fixed reference point $(1,1,1)$. Lower WER, AAS, and PCE and higher BLEU and HV are better. All table entries are single runs. Reported means and population standard deviations aggregate the six requests and are not uncertainty estimates over independent reruns.

\begin{table}[t]
\centering
\begingroup
\renewcommand{\arraystretch}{1.06}
\setlength{\tabcolsep}{2.3pt}
\begin{tabular}{c|rr|rr}
\toprule
& \multicolumn{2}{c|}{Fixed LS} & \multicolumn{2}{c}{EA-SURF} \\
Preference & WER (\%) $\downarrow$ & AAS (ms) $\downarrow$ & WER (\%) $\downarrow$ & AAS (ms) $\downarrow$ \\
\midrule
$(.25,.75)$ & 7.37 & 32.8 & 7.52 & \textbf{30.6} \\
$(.50,.50)$ & 7.26 & 35.9 & 7.34 & \textbf{34.5} \\
$(.75,.25)$ & 7.24 & 41.7 & 7.30 & \textbf{40.1} \\
\bottomrule
\end{tabular}
\endgroup
\vspace{-0.3cm}
\caption{Bi-objective results from random initialization. The first two rows are fixed-LS specialists reused as EA-SURF anchors. The remaining rows compare the interior preferences.}
\label{tab:twoobj}
\vspace{-0.3cm}
\end{table}

\vspace{-0.2cm}
\subsection{Baselines and implementation}

\noindent\textbf{Inner training and sampling.} We estimate each task objective by averaging token-level cross-entropy over its active non-padding tokens, then combine tasks using \eqref{eq:ls}. Each minibatch draws at least one example from each active task's index pool with dataset-balance exponent 0.6. Zero-weight tasks are removed from the pool and loss when training specialist anchors. Inner training uses AdamW at $10^{-4}$ with zero weight decay, 1000 warm-up steps, gradient clipping at 10, and one gradient-accumulation step.

\vspace{0.1cm}
\noindent\textbf{Scalarization controls.} Fixed LS holds $\w$ constant in~\eqref{eq:ls}. As a two-objective Tchebysheff diagnostic~\cite{miettinen1999nonlinear}, we use the fixed utopia reference $z_m^\star=0$, select the task with the largest weighted gap $w_m(\ell_m(\theta)-z_m^\star)$ at each step, differentiate its unweighted gap, and add the augmentation $\rho\sum_m(\ell_m(\theta)-z_m^\star)$ with $\rho=10^{-3}$.

\vspace{0.1cm}
\noindent\textbf{Preference steering.} EA-SURF uses Algorithm~\ref{alg:easurf} with Euclidean distances in \eqref{eq:target-share}, $\alpha=1$, and $\varepsilon_{\mathrm U}=10^{-6}$. Each outer measurement averages task losses over 300 no-gradient training iterations per distributed rank, then averages across ranks to compute $\mathbf q(\theta^{(t)})$. Held-out metrics are evaluated only after round $T$. EA-SURF shares the sampler and optimizer family of fixed LS. %Its continuation learning-rate schedule differs from that of a single fixed-LS run, as detailed below.

\vspace{-0.2cm}
\subsection{Bi-objective steering diagnostic}

Table~\ref{tab:twoobj} provides a bi-objective ASR--timestamp diagnostic from random initialization rather than the primary compute-matched benchmark. Fixed LS trains for 110k steps. EA-SURF uses a 50k-step probe followed by $T=8$ rounds with $K_0=\cdots=K_7=10{,}000$, and reuses two 110k-step anchors across requests. Across the three interior preferences, EA-SURF spans 0.22 ASR-WER points and 9.5~ms AAS, versus 0.13 points and 8.9~ms for fixed LS. EA-SURF's training losses $(\ell_{\rm asr},\ell_{\rm ts})$ progress from $(.331,.845)$ through $(.279,.862)$ to $(.274,.923)$, confirming movement along the trade-off. A separate pretrained-model diagnostic at $\p=(.25,.75)$ compares LS with active-gap Tchebysheff: fixed LS improves ASR WER by 0.23 points and AAS by 0.6~ms. Aside from its small augmentation, the active-gap method differentiates only the largest task gap, using less multi-task gradient information per step. The average best PCE across the three EA-SURF trajectories is \(0.0250\), a \(59.6\%\) improvement over fixed LS at \(0.0618\).

\begin{table}[t]
\centering
\begingroup
\ninept
\renewcommand{\arraystretch}{1.06}
\setlength{\tabcolsep}{2pt}
\begin{tabular}{c|rr|rr|rr}
\toprule
& \multicolumn{2}{c|}{ASR WER (\%) $\downarrow$}
& \multicolumn{2}{c|}{AAS (ms) $\downarrow$}
& \multicolumn{2}{c}{BLEU $\uparrow$} \\
Request $\p$
& LS & EA
& LS & EA
& LS & EA \\
\midrule
$(.2,.2,.6)$
& 7.38 & 7.49
& 54.3 & \textbf{41.8}
& \textbf{32.89} & 32.54 \\

$(.2,.4,.4)$
& 7.31 & 7.59
& 43.0 & \textbf{35.2}
& \textbf{31.91} & 31.64 \\

$(.2,.6,.2)$
& 7.33 & 7.62
& 38.0 & \textbf{35.8}
& 29.54 & \textbf{31.71} \\

$(.4,.2,.4)$
& 7.20 & 7.24
& 50.9 & \textbf{38.8}
& \textbf{32.76} & 32.53 \\

$(.4,.4,.2)$
& 7.23 & 7.39
& 42.6 & \textbf{37.0}
& 30.04 & \textbf{31.39} \\

$(.6,.2,.2)$
& 7.14 & 7.21
& 51.0 & \textbf{46.6}
& 31.45 & \textbf{31.47} \\
\midrule
Mean $\Delta$
& \multicolumn{2}{c|}{$+0.16$}
& \multicolumn{2}{c|}{$\mathbf{-7.4}$}
& \multicolumn{2}{c}{$\mathbf{+0.45}$} \\
\bottomrule
\end{tabular}
\endgroup
\vspace{-0.3cm}
\caption{Three-objective results with 180k-step training budgets for both methods. EA denotes EA-SURF and $\Delta=\text{EA-SURF}-\text{LS}$.}
\label{tab:main}
\vspace{-0.3cm}
\end{table}

\subsection{Three-objective preference steering}
Fixed LS uses one 180k-step run for each request. Each EA-SURF run starts from its 100k-step LS checkpoint and uses $T=8$ with $K_0=\cdots=K_7=10{,}000$, for the same nominal 180k request-level budget. Three separately trained 80k-step specialists with permutations of $(.98,.01,.01)$ are reused across requests. We report the prespecified terminal $\theta^{(T)}$. The EA-SURF continuations are launched separately, so their learning-rate schedule histories differ from those of the single 180k-step fixed-LS runs.

Across the six requests in Table~\ref{tab:main}, fixed LS with $\w=\p$ spans 7.14--7.38\% ASR WER, 38.0--54.3~ms AAS, and 29.54--32.89 BLEU. Greater timestamp emphasis improves AAS, while greater translation emphasis improves BLEU, revealing their trade-off over a stable ASR range. EA-SURF reduces AAS for every request by 2.2--12.5~ms and lowers its mean from 46.6 to 39.2~ms, while increasing mean ASR WER by 0.16 points and mean BLEU by 0.45. EA-SURF is not designed to pointwise dominate fixed LS, but to reach a different operating point specified by specialist-relative training-objective shares, so moving toward the target can improve one held-out metric while worsening another. A larger $p_m$ likewise need not yield monotonic improvement in the corresponding downstream metric. Across all requests, the scalarization-to-share trajectories show no observed folding and each has a unique interior mismatch minimum, empirically supporting Assumption~\ref{ass:pf-regularity}.

As an internal steering metric, mean PCE decreases from $.240$ for fixed LS to $.091$ for EA-SURF, with population standard deviations $.119$ and $.038$, a 62.2\% reduction. For the endpoint-inclusive HV defined above, each six-request set is augmented with the shared anchors $(7.15,\allowbreak60.8,\allowbreak29.08)$, $(7.51,\allowbreak33.4,\allowbreak27.96)$, and $(7.79,\allowbreak62.1,\allowbreak33.68)$ in $(\mathrm{WER}_{\mathrm{ASR}},\allowbreak\mathrm{AAS},\allowbreak\mathrm{BLEU})$ order. Normalized HV increases from .519 to .612, or 17.8\%, indicating improved endpoint-inclusive trade-off coverage.

\vspace{-0.3cm}
\section{Conclusion}

\vspace{-0.1cm}
In this paper, we propose EA-SURF,   which steers multi-task speech LLMs without directly using the deployment preference as linear-scalarization weights. To achieve this separation, it uses one reusable anchor model per task, each trained predominantly for that task. Specifically, EA-SURF converts the preference into target proportions of the Euclidean distances between the current model’s training-loss vector and the anchor models’ loss vectors, then updates the weights based on the measured mismatch along a warm-started training trajectory. Under the stated regularity assumptions, each interior target has a unique supported realization, and the weight update is stationary exactly when the measured and target shares coincide. Empirically, across six three-objective requests, EA-SURF improves preference alignment, lowers mean AAS by 7.4~ms, and raises endpoint-inclusive normalized HV by 17.8\% over fixed LS. These results of EA-SURF demonstrate improved preference calibration and trade-off coverage compared to baselines.

\bibliographystyle{ICASSP2027_Paper_Templates/IEEEbib}
\bibliography{main}

\end{document}